\pdfoutput=1 
\documentclass[letterpaper]{article} 
\usepackage[utf8]{inputenc} 
\usepackage{aaai2026}  
\usepackage{times}  
\usepackage{helvet}  
\usepackage{courier}  
\usepackage[hyphens]{url}  
\usepackage{graphicx} 
\usepackage{natbib}  
\usepackage{caption} 
\usepackage{algorithm}
\usepackage{algorithmic}

\usepackage{newfloat}
\usepackage{listings}
\DeclareCaptionStyle{ruled}{labelfont=normalfont,labelsep=colon,strut=off} 
\floatstyle{ruled}
\newfloat{listing}{tb}{lst}{}
\floatname{listing}{Listing}
\title{Rules or Character? Scaling Laws for AI Safety Design}
\author{
    Satoshi Takahashi\textsuperscript{\rm 1,2},
    Nobuji Kouno\textsuperscript{\rm 1,2},
    Masaaki Komatsu\textsuperscript{\rm 1,2},
    Ryuji Hamamoto\textsuperscript{\rm 1,2}
}
\affiliations{
    \textsuperscript{\rm 1}AI Medical Engineering Team, RIKEN Center for Advanced Intelligence Project, Tokyo, Japan\\
    \textsuperscript{\rm 2}Division of Medical AI Research and Development, National Cancer Center Research Institute, Tokyo, Japan\\
    satoshi.takahashi.fy@riken.jp
}

\usepackage{bibentry}

\usepackage{booktabs}
\usepackage{multirow} 
\usepackage{placeins}
\usepackage{amsmath}   
\usepackage{amssymb}   
\usepackage{amsthm}

\newtheorem{proposition}{Proposition}

\begin{document}

\maketitle

\begin{abstract}
Artificial Intelligence (AI) safety systems combine character
shaping (e.g., Reinforcement Learning from Human Feedback [RLHF],
Constitutional~AI), which modifies behavioral distributions at
training time, with rule enforcement (e.g., output filters, safety
classifiers), which blocks harmful outputs at inference time, yet
little formal analysis exists on how their optimal balance should
change as deployment scales increase. We introduce a stylized
comparative-statics model that parameterizes safety design as a
resource allocation $\alpha \in [0,1]$ between these two approaches,
incorporating scale-dependent filter degradation, common-mode
failures, and character fragility---the risk that shaped behavior
degrades or collapses under novel conditions. Under a multiplicative
Pareto damage model, we derive closed-form expected harm and
supplement it with tail-risk (CVaR) analysis via Monte Carlo
simulation. Across three scenarios (optimistic, moderate,
pessimistic), the optimal $\alpha^*$ is interior or at the
rules-only boundary and shifts weakly toward character shaping as
deployment scale $T$ grows, from negligible
($\Delta\alpha^* = +0.01$) to pronounced ($\Delta\alpha^* = +0.21$)
depending on scenario. The dominant parameter is the baseline
character fragility rate $p_{\mathrm{frag}}^{(0)}$, which shifts
$\alpha^*$ by $0.50$ across its range---far exceeding the effect of
tail severity, filter quality, or common-mode failure probability.
CVaR and expected-harm optima converge at large~$T$. These results
suggest that safety architecture decisions depend less on deployment
scale per se than on the reliability of character shaping under
distributional shift.
\end{abstract}


\section{Introduction}

It is natural to aspire to a world in which accidents never occur.
Yet reducing the probability of undesirable outcomes to exactly zero
appears to be impossible in principle. This is because the entities
that act in the world --- whether human beings or artificial agents
--- can be understood as fundamentally probabilistic: their behavior
is sampled from distributions rather than determined by fixed rules,
and any distribution with nonzero variance will occasionally produce
outcomes in its tails. As long as this is the case, some nonzero
probability of harmful action will persist regardless of the
safeguards in place \cite{kaplan1981quantitative}.

This observation takes on particular urgency for AI systems.
A well-trained language model agent, deployed
to assist users in medical, legal, or financial domains, need not
necessarily be malicious to cause harm. It needs only to encounter a
situation that falls outside the effective range of its training
--- a novel query, an ambiguous context, a subtle distributional
shift --- and produce an output that, while generated in good faith,
leads to consequences ranging from minor inconvenience to
catastrophic damage. The same error that is trivial when a healthy
person asks a casual health question may prove devastating when a
critically ill patient relies on the answer for treatment decisions.
This context-dependent amplification of harm severity, combined with
the irreducible probabilistic nature of model outputs, means that
even a well-aligned agent will occasionally produce harmful outcomes
as a matter of statistical inevitability.

If eliminating risk entirely is unattainable, then the practical
objective becomes risk management: minimizing both the probability
and severity of harmful outcomes, as is standard practice in
medicine, aviation, and nuclear safety
\cite{rasmussen1997risk, aven2016risk}. However, the deployment of
AI agents introduces conditions that
fundamentally distinguish them from the human actors for which
existing risk management frameworks were designed. A single AI model
can be replicated across millions of concurrent instances, each
engaging in independent interactions with users. The total number of
actions taken by a deployed AI system can exceed, within a single
day, the number of consequential decisions a human professional
might make over an entire career. This difference is not merely
quantitative; it is qualitative in its implications for risk. In
human contexts, even imperfect safeguards may suffice because the
total number of trials remains bounded. For AI systems deployed at
scale, the sheer volume of interactions means that even very low
per-interaction failure probabilities can produce frequent harmful
outcomes in aggregate. Furthermore, because replicated instances
share the same model weights, a single vulnerability --- once
discovered or inadvertently triggered --- can affect all instances
simultaneously, creating a risk of correlated common-mode failure
that has no direct analogue in human risk management.
At such scales, familiar per-instance ethical reasoning may
face structural difficulties akin to those identified in
infinite ethics \cite{askell2018pareto}, not because
deployment is infinite, but because the aggregate defies
finite-population intuitions.

How, then, should risk be managed for AI systems operating under
these conditions? The available approaches can be broadly divided
into two categories. The first works on the \emph{interior} of the
acting entity: it seeks to shape the entity's character --- its
dispositions, values, and behavioral tendencies --- so that harmful
actions become intrinsically unlikely. In current AI practice, this
corresponds to training-time interventions such as reinforcement
learning from human feedback \cite{ouyang2022training} and
Constitutional AI \cite{bai2022constitutional}, which modify the
model's underlying output distribution. The second works on the
\emph{exterior}: it imposes rules, filters, and constraints that
intercept harmful outputs after they are generated, without altering
the generative process itself. Runtime safety classifiers, output
filters, and Constitutional Classifiers
\cite{sharma2025constitutional} exemplify this approach.

Each approach carries characteristic vulnerabilities. Character
shaping modifies the model's internal distribution but cannot
guarantee that the shaped behavior will generalize to all deployment
conditions. Recent work on deceptive alignment has demonstrated that
safety training can fail to eliminate undesirable behaviors that
persist through the training process \cite{hubinger2024sleeper},
and more broadly, any shaped distribution may prove fragile when
the model encounters inputs sufficiently far from its training
distribution. We refer to this risk as \emph{character fragility}
--- the possibility that shaped safe behavior reverts to baseline
or worse under novel conditions. Rule enforcement, conversely,
operates independently of the model's internal state but faces a
more fundamental limitation: no finite set of rules can anticipate
every situation an agent may encounter. This echoes a core insight
from systems safety engineering --- Reason's ``Swiss cheese model''
\cite{reason1990human}  demonstrates that every defensive layer contains
gaps, and Leveson's systems-theoretic analysis
\cite{leveson2011engineering} argues that safety constraints
inevitably become inadequate as systems operate in conditions
unforeseen by their designers. Filters are designed against foreseen
failure modes, yet deployed systems inevitably encounter inputs,
contexts, and edge cases that their designers did not --- and in
principle could not --- predict. As deployment scale grows, the
probability that such unforeseen situations arise increases simply
because more diverse interactions are attempted, and when a gap in
the rule set is exposed, it affects all instances sharing the same
filter configuration simultaneously. This is not primarily a
problem of adversarial exploitation; it is a consequence of the
inherent incompleteness of any rule-based safeguard system
\cite{perrow1984normal}.

This distinction --- between shaping character and enforcing rules
--- echoes a long-standing tension in moral philosophy between
virtue ethics, which emphasizes the cultivation of internal
dispositions \cite{hursthouse1999, aristotle_ne}, and deontological
ethics, which emphasizes adherence to external rules. In human
societies, most normative systems blend both elements. The same is
true of modern AI safety architectures, which typically combine
training-time shaping with runtime safeguards. While scaling laws
for model \emph{capability} are now well characterized
\cite{kaplan2020scaling}, and concrete taxonomies of safety failure
modes have been proposed \cite{amodei2016concrete}, scaling laws for
safety \emph{design} --- how the optimal mix of safeguards should
change with scale --- remain largely unexamined.

In this paper, we introduce a stylized comparative-statics model
that formalizes the design space of AI safety as a continuous
spectrum between character shaping and rule enforcement,
parameterized by a resource allocation coefficient
$\alpha \in [0,1]$. We ask: \emph{how does the optimal
$\alpha^*$ --- the allocation that minimizes expected harm ---
change as deployment scale $T$ increases, and what parameters
most strongly determine this optimum?} Through analytical
derivation and Monte Carlo simulation across three scenarios
(optimistic, moderate, and pessimistic), we find that $\alpha^*$
is consistently interior and shifts weakly toward character shaping
with scale, though the magnitude of this shift varies substantially
across scenarios. The single most influential determinant of the
optimal design proves to be neither deployment scale nor tail-risk
severity, but the baseline rate of character fragility --- the
probability that shaped behavior fails under novel conditions. These
findings suggest that safety architecture decisions are governed
less by how large a system is deployed than by how reliably its
character shaping generalizes beyond training conditions.

\section{Related Work}

Our work connects several lines of research spanning AI safety,
systems safety engineering, and risk analysis.

\paragraph{Training-time safety interventions.}
Reinforcement learning from human feedback (RLHF) trains models to
align with human preferences through reward modeling
\cite{christiano2017deep, ouyang2022training}. Constitutional AI
extends this by using written principles to guide self-critique and
revision \cite{bai2022constitutional}. Research on moral
self-correction has shown that sufficiently large language models can
reduce harmful outputs when instructed to do so, suggesting that
models can internalize normative concepts through training
\cite{ganguli2023moral}. These developments motivate our
formalization of training-time shaping as one endpoint of the safety
design spectrum.

\paragraph{Runtime safety mechanisms.}
Constitutional Classifiers defend against harmful outputs by training
classifiers on synthetic data derived from normative principles
\cite{sharma2025constitutional}. While such filters can substantially
reduce harmful output rates, their effectiveness is bounded by the
designers' ability to anticipate failure modes --- a limitation we
formalize through the filter quality ceiling $\varepsilon_{\min}$ in
our model. More broadly, any finite set of runtime rules faces an
inherent coverage problem: deployed systems will inevitably encounter
inputs and contexts that fall outside the designers' foresight.

\paragraph{Character fragility and distributional shift.}
A central parameter in our model is the rate at which character
shaping fails under novel conditions. This risk has two distinct
manifestations in the literature. First, research on deceptive
alignment has shown that models can learn to behave safely during
training while retaining misaligned objectives that manifest under
specific triggers \cite{hubinger2024sleeper}. Second, the broader
machine learning literature on distributional shift documents that
model performance can degrade substantially when deployment
conditions diverge from training conditions
\cite{quinonero2009dataset}. Practical detection of such
out-of-distribution inputs remains an active area of research
\cite{hendrycks2017baseline}. Our model abstracts both phenomena
through a single fragility parameter $p_{\mathrm{frag}}(\alpha)$,
which encompasses intentional deception and unintentional
out-of-distribution failure as limiting cases.

\paragraph{Systems safety and normal accidents.}
Our framework draws on foundational concepts from systems safety
engineering. Reason's Swiss cheese model \cite{reason1990human}
holds that every defensive layer contains gaps and that safety
emerges from stacking layers with independent failure modes. We
formalize this intuition by modeling character shaping and rule
enforcement as complementary layers whose optimal balance depends on
deployment conditions. Perrow's theory of normal accidents
\cite{perrow1984normal} argues that in complex, tightly coupled
systems, accidents are not anomalies but inevitable consequences of
system structure --- a perspective that motivates our analysis of
how well-intentioned agents produce harmful outcomes through the
tails of their behavioral distributions. Leveson's systems-theoretic
approach \cite{leveson2011engineering} further argues that safety
constraints become inadequate as systems operate beyond their
designers' assumptions, which in our model corresponds to the
scale-dependent degradation of filter effectiveness.

\paragraph{Tail risk and heavy-tailed damage distributions.}
Our multiplicative Pareto damage model is motivated by empirical
evidence that harm severity in technological systems often follows
heavy-tailed distributions. Edwards et al.~(\citeyear{edwards2016hype})
showed that cybersecurity data breach damages exhibit properties
intermediate between log-normal and power-law distributions, while
Maillart and Sornette~(\citeyear{maillart2010heavy}) reported
heavy-tailed cyber-risk distributions. Software defect costs are
known to increase by orders of magnitude depending on the phase
of discovery \cite{boehm1981software}. While direct evidence for
AI incident damages following Pareto distributions remains limited,
these analogies from related domains provide plausible anchors for
our tail exponent $\alpha_{\mathrm{PL}} \in [2.0, 3.0]$.

\paragraph{Virtue ethics and AI alignment.}
The character--rule distinction echoes the virtue ethics vs.\
deontology debate in moral philosophy \cite{hursthouse1999,
aristotle_ne}. Noller~(\citeyear{noller2026artificial}) analyzes
Constitutional AI through an Aristotelian lens; our work differs in
examining engineering consequences rather than normative status.

\section{Formal Framework}

We formalize the space of AI safety design as a continuous spectrum
between character shaping and rule enforcement. The model presented
here is a \emph{stylized comparative-statics model}: it does not
aim to precisely replicate real-world AI safety systems, but rather
to analyze the structural properties of the tradeoff between the
two approaches. Accordingly, our conclusions take the form of
qualitative tendencies and boundary conditions rather than
quantitative design recommendations.

\subsection{Action Space and Harm}

Let the action space be $\mathcal{A} \subset \mathbb{R}$ (one-dimensional).
Each action $a \in \mathcal{A}$ is associated with a safety score
$s(a) = a$. A harmful outcome occurs when $s(a) < \tau$ for a safety
threshold $\tau < 0$. Prior to any safety intervention, an entity's
actions are drawn from a baseline distribution
$P_0 = \mathcal{N}(\mu_0, \sigma_0^2)$.

We define harm under two models. \textbf{Model~A} (deterministic
damage) sets the harm as the distance below the threshold:
\begin{equation}
  h(a) = (\tau - a)_+ = \max(0,\; \tau - a).
\end{equation}
\textbf{Model~B} (multiplicative Pareto damage) reflects the
observation that, for well-intentioned agents, the same error can
produce vastly different consequences depending on the context in
which it occurs --- a medical misstatement that is harmless in casual
conversation may prove catastrophic when relied upon for treatment
decisions. We capture this context-dependent amplification through a
multiplicative structure:
\begin{equation}\label{eq:model-b}
  h(a) = (\tau - a)_+ \times X, \quad X \sim \mathrm{Pareto}(1,\;\alpha_{\mathrm{PL}}),
\end{equation}
where $(\tau - a)_+$ is the \emph{action magnitude} (how far the
action exceeds the threshold) and $X$ is a \emph{context multiplier}
drawn independently. Because $X$ and $(\tau - a)_+$ are independent,
the expectation factorizes:
\begin{equation}
  \begin{aligned}
    \mathbb{E}[h] &= \mathbb{E}[(\tau - a)_+] \times \mathbb{E}[X],\\
    \text{where } \mathbb{E}[X] &= \frac{\alpha_{\mathrm{PL}}}{\alpha_{\mathrm{PL}} - 1}
    \;\;(\alpha_{\mathrm{PL}} > 1).
  \end{aligned}
\end{equation}
As $\alpha_{\mathrm{PL}} \to \infty$, $\mathbb{E}[X] \to 1$ and
Model~B reduces continuously to Model~A. Thus Model~A is a special
case of Model~B, and the two can be treated within a unified framework.

The tail exponent $\alpha_{\mathrm{PL}}$ governs the heaviness of
the damage distribution. When $\alpha_{\mathrm{PL}} > 2$, both
expectation and variance of $X$ are finite. When
$1 < \alpha_{\mathrm{PL}} \leq 2$, the expectation is finite but
the variance diverges, making CVaR estimation slow to converge.
When $\alpha_{\mathrm{PL}} \leq 1$, both the expected harm and CVaR
diverge, and only quantile-based measures ($\mathrm{VaR}_\beta$)
remain meaningful. Our sensitivity analyses focus on
$\alpha_{\mathrm{PL}} \geq 1.5$.
Values of $\alpha_{\mathrm{PL}} \in [2.0, 3.0]$ are consistent with
heavy-tailed damage estimates from related domains, including
cybersecurity breach data \cite{edwards2016hype, maillart2010heavy}
and software defect cost distributions \cite{boehm1981software}.

\subsection{The Character--Rule Spectrum}

We parameterize safety design by a mixing coefficient
$\alpha \in [0,1]$, interpreted as the fraction of safety resources
allocated to character shaping. $\alpha = 0$ denotes pure rule
enforcement; $\alpha = 1$ denotes pure character shaping.

\paragraph{Character shaping.}
Higher $\alpha$ shifts the action distribution toward safety and
reduces its variance:
\begin{equation}
  \mu(\alpha) = \mu_0 + \alpha \cdot \Delta\mu, \qquad
  \sigma(\alpha) = \sigma_0 \bigl(1 - \alpha(1 - r_\sigma)\bigr),
\end{equation}
where $\Delta\mu > 0$ is the safety shift achievable through
training-time intervention and $r_\sigma \in (0,1)$ controls the
degree of variance reduction. The shaped distribution is
$P_\alpha = \mathcal{N}(\mu(\alpha),\;\sigma(\alpha)^2)$.

\paragraph{Filter quality.}
Higher $\alpha$ reduces resources available for filter development,
degrading filter quality. Filters have a technology-imposed
performance ceiling $\varepsilon_{\min}$ that cannot be surpassed
even with full resource investment:
\begin{equation}\label{eq:eps-base}
  \varepsilon_{\mathrm{base}}(\alpha) = \varepsilon_{\min}
  + (\varepsilon_{\max,\mathrm{base}} - \varepsilon_{\min}) \cdot \alpha^k.
\end{equation}
At $\alpha = 0$, $\varepsilon = \varepsilon_{\min}$ (best achievable
filter); at $\alpha = 1$,
$\varepsilon = \varepsilon_{\max,\mathrm{base}}$ (minimal filter
quality). The value of $\varepsilon_{\min}$ is anchored to recent results on
Constitutional Classifiers, which achieved jailbreak success rates
of approximately 4.4\% after extensive red-teaming
\cite{sharma2025constitutional}. While this adversarial benchmark
differs from the benign edge-case setting motivating our model, it
provides a conservative lower bound on achievable filter quality.

\subsection{Deployment Scale and Edge-Case Pressure}

We decompose the total interaction count $T$ into a component that
drives harm accumulation and a component that drives filter
degradation and systemic vulnerability discovery:
\begin{equation}
  M = \rho_{\mathrm{edge}} \cdot T,
  \qquad A(M) = M / M_{\mathrm{ref}},
\end{equation}
where $\rho_{\mathrm{edge}}$ is the fraction of interactions that
constitute edge cases --- inputs, contexts, or situations not
anticipated during filter design --- and $M_{\mathrm{ref}}$ is a
reference scale. This decomposition separates $T$ as a linear scale
factor for harm from $M$ as the driver of filter degradation and
systemic vulnerability probability.

\paragraph{Scale-dependent filter degradation.}
As deployment scale grows, filters encounter unforeseen edge cases
with increasing probability, raising the effective leakage rate:
\begin{equation}\label{eq:eps-full}
  \begin{aligned}
    \varepsilon(\alpha, M)
    &= \varepsilon_{\mathrm{base}}(\alpha) \\
    &\quad + \bigl(\varepsilon_{\mathrm{ceiling}}
    - \varepsilon_{\mathrm{base}}(\alpha)\bigr) \\
    &\quad \cdot (1 - e^{-\beta_d \, A(M)}) \cdot d_0,
  \end{aligned}
\end{equation}
where $\beta_d$ is the sensitivity of pattern discovery to
deployment scale and $d_0 \in (0,1]$ is a \emph{diffusion fraction}
--- the proportion of discovered vulnerability patterns that
propagate broadly. Note that $d_0$ governs the ultimate reach of
discovered patterns, not their speed of propagation: as
$A(M) \to \infty$, the leakage rate approaches
$\varepsilon_{\mathrm{base}}(\alpha)
+ (\varepsilon_{\mathrm{ceiling}} - \varepsilon_{\mathrm{base}}(\alpha))
\cdot d_0$, which is strictly less than $\varepsilon_{\mathrm{ceiling}}$
when $d_0 < 1$.

\paragraph{Common-mode failure (CMF).}
A CMF occurs when a single vulnerability compromises
all deployed instances simultaneously --- for example, an unforeseen
blind spot shared by all instances due to identical model weights.
The probability of such an event increases with edge-case pressure:
\begin{equation}\label{eq:q}
  q(M) = (1 - e^{-\beta_q \, A(M)}) \cdot e_0,
\end{equation}
where $\beta_q$ is the sensitivity of CMF discovery and $e_0$ is
the probability that a discovered systemic vulnerability is actually
triggered. Crucially, $q(M)$ does not depend on $\alpha$: CMF arises
from architectural properties of the filter layer and deployment
infrastructure, not from the quality of character shaping. The
protective effect of character shaping during CMF is instead
reflected in the reduced post-CMF harm (see Expected Harm below).

\subsection{Character Fragility}\label{sec:fragility}

As reliance on character shaping increases, so does the risk that
the shaped behavior proves fragile --- degrading or collapsing
when the model encounters conditions outside its effective
training range. We refer to this risk as \emph{character fragility}
and model it through a distribution-switching mechanism:
\begin{equation}\label{eq:pfrag}
  p_{\mathrm{frag}}(\alpha) = p_{\mathrm{frag}}^{(0)} \cdot \alpha^n,
\end{equation}
where $p_{\mathrm{frag}}^{(0)}$ is the \emph{baseline character
fragility rate} --- the maximum per-interaction probability of
character failure (attained at $\alpha = 1$) --- and $n$ is an
exponent parameter (default $n = 2$; sensitivity analysis over
$n \in \{0.5, 1, 2, 3\}$). When fragility manifests
(probability $p_{\mathrm{frag}}(\alpha)$), the action distribution
switches from $P_\alpha$ to a fragility distribution
$P_{\mathrm{frag}} = \mathcal{N}(\mu_{\mathrm{frag}},\;
\sigma_{\mathrm{frag}}^2)$, with $\mu_{\mathrm{frag}} \leq \mu_0$.

Importantly, $p_{\mathrm{frag}}^{(0)}$ is an intrinsic property
of the trained model --- the fraction of input space where shaped
behavior fails to hold --- analogous to a manufacturing defect rate
determined by the production process, not by the number of units
produced. What changes with deployment scale $T$ is the aggregate
number of fragility manifestations, not the per-interaction rate.
This is why $p_{\mathrm{frag}}$ is $T$-independent in our model.
As we show in the Simulation Results,
$p_{\mathrm{frag}}^{(0)}$ proves to be the single most influential
parameter in determining the optimal safety design.

The parameter $p_{\mathrm{frag}}$ subsumes two qualitatively
distinct failure modes. The first is \emph{intentional deception}
(deceptive alignment), in which a model learns to behave safely
during training while retaining misaligned objectives
\cite{hubinger2024sleeper}. The second is \emph{distributional
fragility}, in which shaped safe behavior simply fails to transfer
to novel conditions \cite{quinonero2009dataset}. Both are monotone
increasing in $\alpha$ and are therefore jointly captured by the
$\alpha^n$ functional form, but their internal mechanisms differ.
Decomposing their relative contributions is beyond the scope of
this model and is left to future work.

The setting of $\mu_{\mathrm{frag}}$ depends on which failure mode
dominates: $\mu_{\mathrm{frag}} < \mu_0$ (worse than baseline)
reflects intentional deception, while
$\mu_{\mathrm{frag}} \approx \mu_0$ (baseline reversion) reflects
distributional fragility. Our scenario analysis examines both cases.

When fragility manifests, the filter may be less effective at
detecting the resulting behavior, which can differ in character
from ordinary harmful outputs. We capture this through
$\varepsilon_{\mathrm{frag}} = \min(\text{factor} \times
\varepsilon(\alpha, M),\; 1.0)$, where the factor is
scenario-dependent.

\subsection{Expected Harm}\label{sec:eharm}

Rather than decomposing expected harm into separate probability and
conditional-severity terms (which introduces weighting errors in the
presence of mixture distributions), we define per-interaction
expected harm directly.

\paragraph{Base quantities.}
For Model~A, the unconditional per-interaction expected harm under
distribution $P_\alpha$ admits the closed form:
\begin{equation}\label{eq:g-det}
  g_\alpha^{\mathrm{det}} = (\tau - \mu(\alpha))\,
  \Phi\!\Bigl(\frac{\tau - \mu(\alpha)}{\sigma(\alpha)}\Bigr)
  + \sigma(\alpha)\,
  \phi\!\Bigl(\frac{\tau - \mu(\alpha)}{\sigma(\alpha)}\Bigr),
\end{equation}
where $\Phi$ and $\phi$ are the standard normal CDF and PDF,
respectively. The quantity $g_{\mathrm{frag}}^{\mathrm{det}}$ is
defined analogously for $P_{\mathrm{frag}}$. For Model~B, the
Pareto context multiplier scales these quantities uniformly:
\begin{equation}\label{eq:g-pareto}
  g_\alpha = g_\alpha^{\mathrm{det}} \times
  \frac{\alpha_{\mathrm{PL}}}{\alpha_{\mathrm{PL}} - 1}, \qquad
  g_{\mathrm{frag}} = g_{\mathrm{frag}}^{\mathrm{det}} \times
  \frac{\alpha_{\mathrm{PL}}}{\alpha_{\mathrm{PL}} - 1}.
\end{equation}
Because $\mathbb{E}[X]$ is independent of $\alpha$, the
$\arg\min_\alpha$ of expected harm is identical under Models~A and~B.

\paragraph{Per-interaction expected harm.}
Under normal operation (no CMF):
\begin{equation}\label{eq:L-normal}
\begin{aligned}
  L_{\mathrm{normal}}(\alpha, M)
  &= (1 - p_{\mathrm{frag}}(\alpha))
     \cdot \varepsilon(\alpha, M) \cdot g_\alpha \\
  &\quad + p_{\mathrm{frag}}(\alpha)
     \cdot \varepsilon_{\mathrm{frag}}
     \cdot g_{\mathrm{frag}}.
\end{aligned}
\end{equation}
Under CMF (filters fully disabled):
\begin{equation}\label{eq:L-CMF}
  L_{\mathrm{CMF}}(\alpha) = (1 - p_{\mathrm{frag}}(\alpha))
  \cdot g_\alpha + p_{\mathrm{frag}}(\alpha) \cdot g_{\mathrm{frag}}.
\end{equation}
Note the absence of $\varepsilon$ terms in $L_{\mathrm{CMF}}$:
filters are inoperative during CMF. However, character shaping
persists because it is embedded in the model weights. Since
$g_\alpha < g_0$ for $\alpha > 0$, character shaping automatically
reduces post-CMF harm without requiring any additional parameter.

\paragraph{System-level expected harm.}
\begin{equation}\label{eq:E-harm}
\begin{aligned}
  E_{\mathrm{harm}}(\alpha, T, M)
  &= T \cdot \bigl[
     (1 - q(M)) \cdot L_{\mathrm{normal}}(\alpha, M) \\
  &\quad + q(M) \cdot L_{\mathrm{CMF}}(\alpha) \bigr].
\end{aligned}
\end{equation}
Here $T$ acts as a linear scale factor for harm, while $M$ drives
filter degradation and CMF probability through
Equations~\eqref{eq:eps-full} and~\eqref{eq:q}.

\subsection{Comparative Statics}

\begin{proposition}[Filter technology improvement lowers $\alpha^*$]
\label{prop:eps-min}
$\partial \alpha^* / \partial \varepsilon_{\min} > 0$. That is, an
improvement in filter technology (lower $\varepsilon_{\min}$)
reduces the optimal character weight.
\end{proposition}

\begin{proof}[Proof sketch]
Define $K = (1 - e^{-\beta_d A(M)}) \cdot d_0$, which is
independent of~$\alpha$. Then
$\varepsilon(\alpha, M) =
\varepsilon_{\mathrm{base}}(\alpha)(1 - K) +
\varepsilon_{\mathrm{ceiling}} \cdot K$, and since
$\varepsilon_{\mathrm{base}}(\alpha) =
\varepsilon_{\min}(1 - \alpha^k) +
\varepsilon_{\max,\mathrm{base}} \cdot \alpha^k$,
\begin{equation}
  \frac{\partial \varepsilon}{\partial \varepsilon_{\min}}
  = (1 - \alpha^k)(1 - K).
\end{equation}
This is maximized at $\alpha = 0$ (value $1 - K$) and vanishes at
$\alpha = 1$ (no resources allocated to filters, hence no
sensitivity to filter technology). A decrease in $\varepsilon_{\min}$
therefore reduces $L_{\mathrm{normal}}$ more at low~$\alpha$ than at
high~$\alpha$, shifting the minimum of $E_{\mathrm{harm}}$ leftward.
A formal proof via the implicit function theorem is
straightforward and omitted for brevity.
\end{proof}

\paragraph{Predicted tendency.}
As $M$ increases, both $\varepsilon(\alpha, M)$ and $q(M)$ rise.
The former penalizes filter-reliant (low-$\alpha$) designs through
$L_{\mathrm{normal}}$; the latter amplifies $L_{\mathrm{CMF}}$, in
which character shaping (via $g_\alpha < g_0$) provides the only
remaining protection. Both effects push $\alpha^*$ upward. However,
at high $\alpha$ the fragility cost $p_{\mathrm{frag}}(\alpha) \cdot
g_{\mathrm{frag}}$ also grows, potentially offsetting these effects
in high-fragility regimes. The conditions under which $\alpha^*(T)$
is monotone non-decreasing are identified empirically through phase
diagrams in the Simulation Results.

\subsection{Tail Risk (CVaR)}

Expected harm captures average performance but may understate
catastrophic scenarios. We therefore also compute the Conditional
Value-at-Risk at level~$\beta$:
\begin{equation}
  \mathrm{CVaR}_\beta(\alpha, T, M) = \mathbb{E}[\mathrm{harm}
  \mid \mathrm{harm} > \mathrm{VaR}_\beta],
\end{equation}
following the framework of \citet{rockafellar2000optimization},
estimated via count-level Monte Carlo simulation. Rather than
sampling each of the $T$ interactions individually (which is
computationally infeasible at $T = 10^8$), we sample aggregate
counts: the number of fragility-manifesting interactions from
$\mathrm{Binomial}(T, p_{\mathrm{frag}}(\alpha))$, the number of
harmful events from the appropriate binomial, and individual damages
from the conditional harm distribution (with Pareto context
multipliers under Model~B). This reduces the per-replication cost
from $O(T)$ to $O(n_{\mathrm{harm}})$, where $n_{\mathrm{harm}}
\ll T$.

The direction of $\alpha^*$ under CVaR is not determined a priori.
When the tail is dominated by CMF events (in which filters are
disabled), character shaping provides the only mitigation, pushing
$\alpha^*_{\mathrm{CVaR}}$ upward. When the tail is instead driven
by per-incident severity (especially under Model~B with small
$\alpha_{\mathrm{PL}}$), filter effectiveness may become more
valuable. The relative strength of these channels depends on $q(M)$,
$\alpha_{\mathrm{PL}}$, and $p_{\mathrm{frag}}^{(0)}$, and is
resolved empirically in the Simulation Results.

For $\alpha_{\mathrm{PL}} = 2.0$ (the Pessimistic scenario), the
variance of the Pareto distribution is infinite, causing CVaR
estimates to converge more slowly than in the finite-variance case.
We address this by running five independent seeds and reporting
bootstrap 95\% confidence intervals for all CVaR estimates.

\subsection{Scenario-Based Parameterization}\label{sec:scenarios}

Rather than claiming empirically calibrated parameter values, we
treat all parameters as \emph{scenario anchors} and evaluate the
model under three regimes:
\textbf{Optimistic} (strong character shaping, high-quality filters,
low fragility), \textbf{Moderate} (intermediate values), and
\textbf{Pessimistic} (weak character shaping, poor filters, high
fragility and heavy damage tails). Full parameter tables are provided
in Appendix~A. By comparing $\alpha^*(T)$ across all
three scenarios, we assess the robustness of qualitative conclusions
and identify any regime-dependent reversals.

\section{Simulation Results}\label{sec:results}

We evaluate the model across three scenarios (Optimistic, Moderate,
Pessimistic) with parameters given in Table~\ref{tab:scenario_params}.
All expected-harm results are computed
analytically via closed-form expressions; CVaR estimates use
count-level Monte Carlo simulation with $M_{\mathrm{sim}} = 10{,}000$
replications (50{,}000 for production figures) and
$\beta = 0.99$.

\subsection{The Optimal Mix Is Interior and Weakly Increasing in
Scale}\label{sec:main-result}

Figure~\ref{fig:alpha-star-T} presents the central result:
$\alpha^*(T)$ across the three scenarios. In all cases, pure
character ($\alpha = 1$) is never optimal; the optimum is either
interior or, in the Pessimistic scenario at small~$T$, at the
rules-only boundary ($\alpha^* = 0$). The optimal character weight increases
weakly with deployment scale~$T$, but the magnitude of this shift
varies substantially across scenarios(Table~\ref{tab:alpha-star}). 

\begin{table}[h]
\centering
\small
\begin{tabular}{@{}lccccc@{}}
\toprule
Scenario & \multicolumn{4}{c}{$\alpha^*(T)$} & $\Delta\alpha^*$ \\
\cmidrule(lr){2-5}
         & $10^2$ & $10^4$ & $10^6$ & $10^8$ & \\
\midrule
Optimistic   & 0.62 & 0.62 & 0.62 & 0.63 & $+0.01$ \\
Moderate     & 0.51 & 0.51 & 0.52 & 0.55 & $+0.04$ \\
Pessimistic  & 0.00 & 0.00 & 0.06 & 0.21 & $+0.21$ \\
\bottomrule
\end{tabular}
\caption{\textbf{Optimal character weight $\alpha^*$ as a function
of deployment scale $T$.} $\Delta\alpha^* = \alpha^*(10^8) -
\alpha^*(10^2)$. The scale effect is negligible under Optimistic
assumptions, modest under Moderate, and pronounced under
Pessimistic conditions, where $\alpha^*$ transitions from the
boundary ($\alpha^* = 0$, pure rules) to an interior solution.}
\label{tab:alpha-star}
\end{table}

\begin{figure}[htbp]
\centering
\includegraphics[width=\columnwidth]{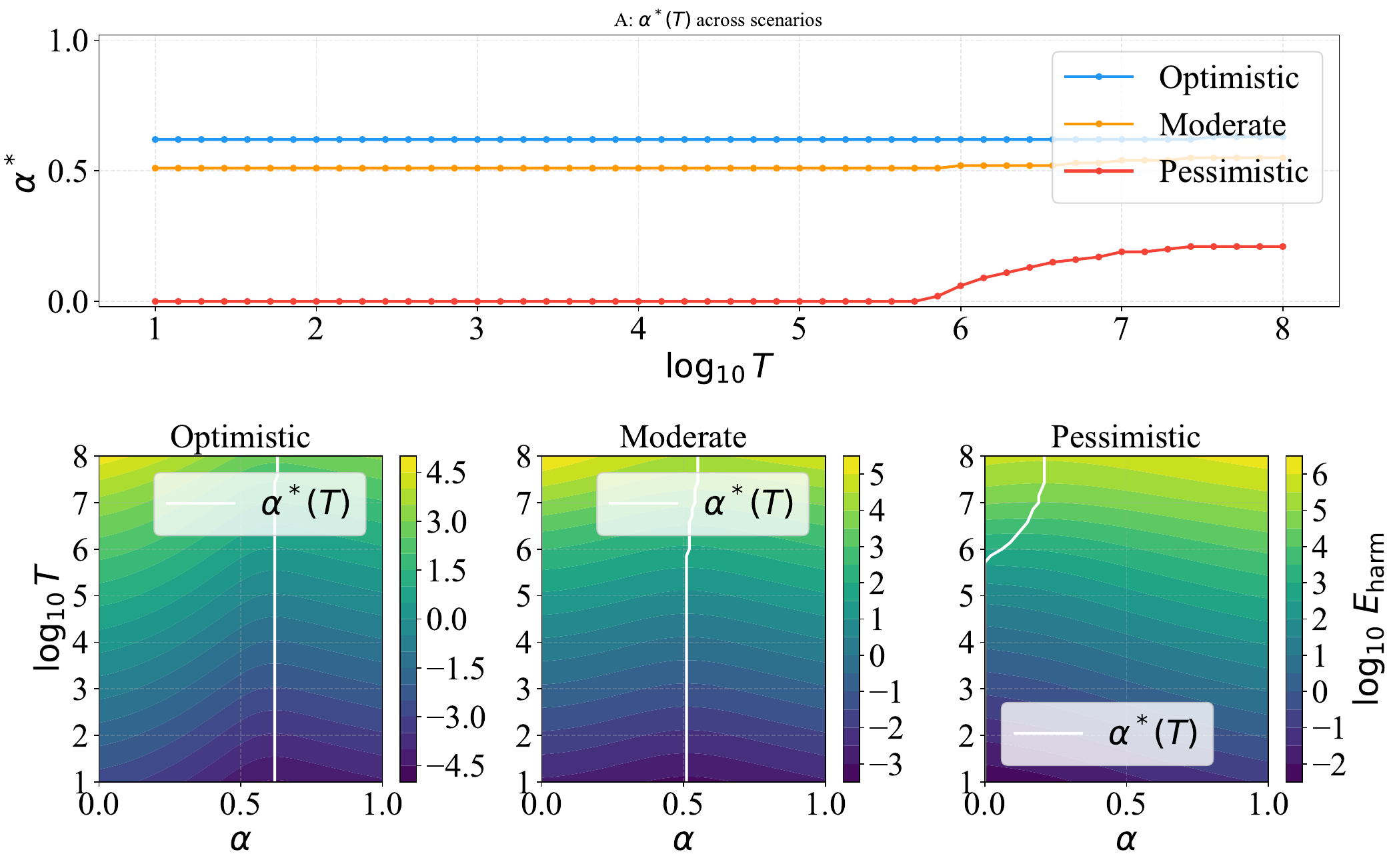}
\caption{\textbf{Optimal character weight $\alpha^*(T)$ across
three scenarios.} \textit{Top:} $\alpha^*$ as a function of
deployment scale $T$. The optimum is weakly non-decreasing in $T$;
pure character is never optimal. The Optimistic scenario yields
a nearly flat $\alpha^* \approx 0.62$ ($\Delta\alpha^* = +0.01$),
while the Pessimistic scenario exhibits a sharp transition from
$\alpha^* = 0$ (pure rules) to an interior solution near
$T \approx 10^{5.5}$ ($\Delta\alpha^* = +0.21$).
\textit{Bottom:} Contour maps of $\log_{10} E_{\mathrm{harm}}
(\alpha, T)$ for each scenario, with the $\alpha^*(T)$ trajectory
overlaid in white. The valley of minimal harm is narrow and
nearly vertical in the Optimistic case, broader and rightward-shifting
in the Moderate case, and sharply kinked in the Pessimistic case,
reflecting the phase transition visible in the top panel.}
\label{fig:alpha-star-T}
\end{figure}

Three qualitatively distinct regimes emerge. Under \textbf{Optimistic}
assumptions (strong character shaping, low fragility), $\alpha^*$ is
essentially flat at approximately~0.62 across six orders of magnitude
in~$T$. This suggests that when character-shaping technology is
sufficiently mature and fragility risk is low, the scaling law
effectively vanishes: the optimal design is insensitive to deployment
scale. Under \textbf{Moderate} assumptions, the scaling effect is
present but modest ($\Delta\alpha^* = +0.04$). Under
\textbf{Pessimistic} assumptions (weak character shaping, high
fragility), $\alpha^*$ begins at the boundary ($\alpha^* = 0.00$,
pure rules) for small~$T$ and transitions sharply to an interior
solution near $T \approx 10^{5.5}$, with a total shift of~$+0.21$.

\paragraph{Analytical characterization of the Pessimistic transition.}
The transition from $\alpha^* = 0$ to an interior solution occurs at
the critical scale $T_{\mathrm{crit}}$ where
$\partial E_{\mathrm{harm}} / \partial \alpha \big|_{\alpha=0} = 0$.
At $\alpha = 0$, the fragility cost vanishes
($p_{\mathrm{frag}}(0) = 0$ for $n \geq 1$), so the condition
reduces to a balance between two forces: the benefit of character
shaping (reducing $g_\alpha$) and the cost of filter degradation
(increasing $\varepsilon_{\mathrm{base}}$). As $M$ grows, two
effects favor character shaping: $\varepsilon(\alpha, M)$ rises
(making filter reliance costlier) and $q(M)$ rises (amplifying the
CMF channel, in which character shaping provides the sole remaining
defense via $g_\alpha < g_0$). At $M = M_{\mathrm{crit}}$, these
effects overcome the filter-degradation cost, and the optimum
detaches from the boundary. Under Pessimistic parameters,
$T_{\mathrm{crit}} \approx 10^{5.5}$, consistent with the observed
transition in Figure~\ref{fig:alpha-star-T}. We emphasize that the specific value $T_{\mathrm{crit}} \approx
10^{5.5}$ is contingent on the Pessimistic parameter settings;
under different parameterizations the transition point shifts
accordingly, though the qualitative phenomenon --- a sharp onset
of interior optimality beyond a critical scale --- is robust across
the parameter space.

\subsection{Phase Diagrams: $\Delta\alpha^* \geq 0$ Throughout the
Explored Parameter Space}\label{sec:phase}

To identify conditions under which $\alpha^*(T)$ might decrease with
scale, we compute phase diagrams over $20 \times 20$ grids of
parameter pairs, plotting
$\Delta\alpha^* = \alpha^*(10^8) - \alpha^*(10^2)$ at each cell.
Table~\ref{tab:phase} summarizes the results across all three
parameter grids: across 1{,}200 cells, $\Delta\alpha^*$ is
non-negative everywhere, with zero negative cells in any grid.

\begin{table}[h]
\centering
\small
\begin{tabular}{lccc}
\toprule
Parameter pair & Range of $\Delta\alpha^*$ & Positive & Negative \\
\midrule
$\Delta\mu \times p_{\mathrm{frag}}^{(0)}$
  & $[0.00,\; 0.67]$ & 400 & 0 \\
$\Delta\mu \times \varepsilon_{\mathrm{ceiling}}$
  & $[0.02,\; 0.12]$ & 400 & 0 \\
$\varepsilon_{\mathrm{ceiling}} \times p_{\mathrm{frag}}^{(0)}$
  & $[0.03,\; 0.06]$ & 400 & 0 \\
\bottomrule
\end{tabular}
\caption{\textbf{Phase diagram summary. Across 1{,}200 cells in
three $20 \times 20$ grids, $\Delta\alpha^*$ is non-negative
everywhere.}}
\label{tab:phase}
\end{table}

Figure~\ref{fig:phase-diagram} shows the
$\Delta\mu \times p_{\mathrm{frag}}^{(0)}$ grid in detail. The
strongest scale effect ($\Delta\alpha^* \approx 0.67$) occurs at
low $\Delta\mu$ and high $p_{\mathrm{frag}}^{(0)}$ (upper left),
where character shaping is weak and fragile, making the system most
sensitive to scale-driven filter degradation.

\begin{figure}[htbp]
\centering
\includegraphics[width=\columnwidth]{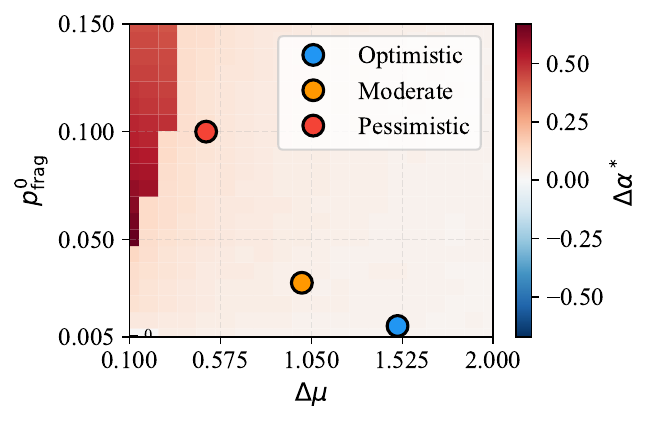}
\caption{\textbf{Phase diagram: scale-induced shift in optimal
design over the $\Delta\mu \times p_{\mathrm{frag}}^{(0)}$
parameter space.} Each cell shows
$\Delta\alpha^* = \alpha^*(10^8) - \alpha^*(10^2)$. Red indicates
that larger scale favors more character-reliant design; blue would
indicate the opposite. Across all 400 cells,
$\Delta\alpha^* \geq 0$. Colored circles mark the three scenarios.
This monotonicity is a structural property of the model (see the
Discussion).}
\label{fig:phase-diagram}
\end{figure}

This result is stronger than initially expected: no parameter
combination in the explored range produces a regime where larger
deployment scale favors more rule-reliant design. We emphasize that
this is a structural property of the current model rather than a
general empirical prediction. In our formulation, deployment scale
$T$ enters exclusively through $M = \rho_{\mathrm{edge}} T$, which
degrades runtime filters (via $\varepsilon(\alpha, M)$) and raises
CMF probability (via $q(M)$). Both channels penalize low-$\alpha$
designs: filter degradation makes rules less reliable, and CMF
eliminates the filter layer entirely, leaving only character shaping
as protection. Since $T$ has no channel through which it degrades
character shaping (i.e., $p_{\mathrm{frag}}(\alpha)$ and $g_\alpha$
are $T$-independent), the monotonicity $\Delta\alpha^* \geq 0$
follows near-tautologically from the model structure. A model in
which deployment scale also increases character fragility (e.g.,
$p_{\mathrm{frag}}(\alpha, T)$) could in principle produce
$\Delta\alpha^* < 0$ regions; we discuss this extension in
the Discussion.

\subsection{Baseline Fragility Rate Dominates the Optimal
Design}\label{sec:pfrag}

Table~\ref{tab:sensitivity} reports the sensitivity of
$\alpha^*(T = 10^6)$ to each model parameter individually. The
baseline fragility rate $p_{\mathrm{frag}}^{(0)}$ dominates: its
effect on $\alpha^*$ ($-0.50$) is nearly twice that of the next
most influential parameter ($\Delta\mu$ at $-0.27$), and all
remaining parameters shift $\alpha^*$ by less than $0.10$. The
effects of $\Delta\mu$ and $r_\sigma$ are discussed in the
Discussion.

\begin{table}[h]
\centering
\small
\begin{tabular}{lcc}
\toprule
Parameter & Swept range & $\Delta\alpha^*$ at $T = 10^6$ \\
\midrule
$p_{\mathrm{frag}}^{(0)}$ & $0.005 \to 0.40$ & $-0.50$ \\
$\Delta\mu$ & $0.2 \to 2.0$ & $-0.27$ \\
$r_\sigma$ & $0.4 \to 1.0$ & $+0.21$ \\
$n_{\mathrm{frag}}$ & $0.5 \to 4.0$ & $+0.09$ \\
$\varepsilon_{\min}$ & $0.005 \to 0.20$ & $+0.07$ \\
$\rho_{\mathrm{edge}}$ & $0.005 \to 0.40$ & $+0.02$ \\
$e_0$ & $0.05 \to 0.95$ & $+0.01$ \\
$\varepsilon_{\mathrm{ceiling}}$ & $0.10 \to 0.80$ & $0.00$ \\
\bottomrule
\end{tabular}
\caption{\textbf{Sensitivity of $\alpha^*(T=10^6)$ to individual
parameters. $p_{\mathrm{frag}}^{(0)}$ is the dominant lever by a
wide margin.}}
\label{tab:sensitivity}
\end{table}

Figure~\ref{fig:pfrag-sensitivity} shows this dominant parameter in
detail: $\alpha^*$ decreases monotonically from $0.70$ (at
$p_{\mathrm{frag}}^{(0)} = 0.005$) to $0.20$ (at
$p_{\mathrm{frag}}^{(0)} = 0.40$) --- a swing of $0.50$.

\begin{figure}[htbp]
\centering
\includegraphics[width=\columnwidth]{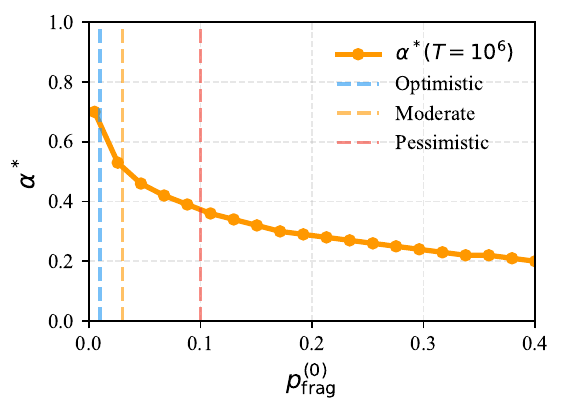}
\caption{\textbf{Baseline fragility rate $p_{\mathrm{frag}}^{(0)}$
is the dominant determinant of optimal design.} The optimal
character weight $\alpha^*(T = 10^6)$ decreases monotonically from
$0.70$ to $0.20$ as $p_{\mathrm{frag}}^{(0)}$ increases from
$0.005$ to $0.40$ (Moderate scenario baseline, varying only
$p_{\mathrm{frag}}^{(0)}$). Dashed vertical lines indicate the
default $p_{\mathrm{frag}}^{(0)}$ values for each scenario.}
\label{fig:pfrag-sensitivity}
\end{figure}

This result carries a clear practical implication: the most
consequential input to safety architecture design is the estimated
reliability of character shaping under distributional shift. If
fragility can be kept below approximately 5\%, the optimal design
allocates a majority of resources to character shaping; above 10\%,
the optimum shifts decisively toward rules.

\subsection{Proposition~1: Filter Technology Improvement Lowers
$\alpha^*$}\label{sec:prop1-result}

Figure~\ref{fig:prop1} confirms Proposition~\ref{prop:eps-min}
numerically: across all three scenarios,
$\alpha^*(T = 10^6)$ increases monotonically with
$\varepsilon_{\min}$, verifying that
$\partial \alpha^* / \partial \varepsilon_{\min} > 0$. As filter
technology improves (lower $\varepsilon_{\min}$), the optimal design
shifts toward greater reliance on rules. The effect is modest in
absolute magnitude ($\Delta\alpha^* = +0.07$ over
$\varepsilon_{\min} \in [0.005, 0.20]$) but consistent in sign
across all scenarios.

\begin{figure}[htbp]
\centering
\includegraphics[width=\columnwidth]{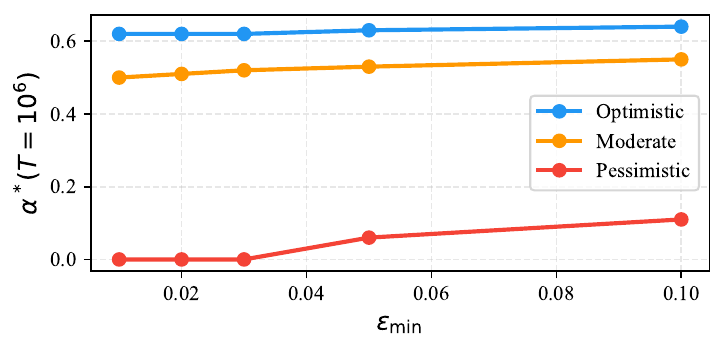}
\caption{\textbf{Numerical verification of Proposition~\ref{prop:eps-min}:
improving filter technology lowers $\alpha^*$.}
$\alpha^*(T = 10^6)$ as a function of the filter quality ceiling
$\varepsilon_{\min}$ across all three scenarios. In every case,
$\alpha^*$ increases monotonically with $\varepsilon_{\min}$,
confirming $\partial \alpha^* / \partial \varepsilon_{\min} > 0$.
As filter technology improves (lower $\varepsilon_{\min}$), the
optimal design shifts toward greater reliance on rules. The effect
is consistent in sign across scenarios, though modest in absolute
magnitude ($\Delta\alpha^* = +0.07$ over
$\varepsilon_{\min} \in [0.005, 0.20]$).}
\label{fig:prop1}
\end{figure}

\subsection{Tail Risk: CVaR Confirms Expected-Harm
Optimum}\label{sec:cvar-results}

The optimal safety design proves robust to the choice of risk
criterion across both damage models.

Under Model~A (deterministic damage), the CVaR-optimal $\alpha^*$
converges to the expected-harm-optimal $\alpha^*$ for $T \gtrsim
10^5$ (Figure~\ref{fig:cvar-comparison}). At smaller~$T$,
Monte~Carlo noise produces fluctuations of $\pm 0.10$ around the
expected-harm optimum, but no systematic divergence is observed.

\begin{figure}[t]
\centering
\includegraphics[width=\columnwidth]{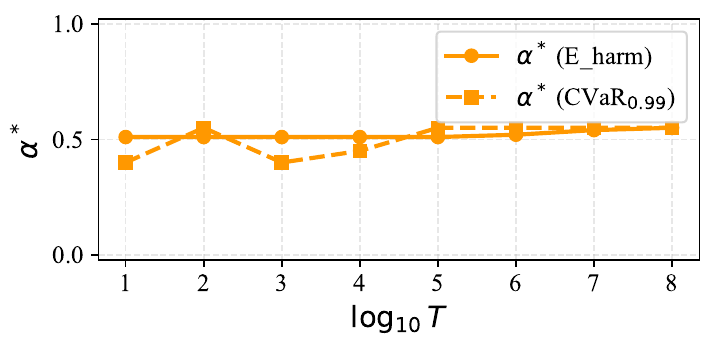}
\caption{\textbf{CVaR-based and expected-harm-based optima converge
at moderate-to-large deployment scale.} Optimal $\alpha^*$ under
the expected-harm criterion (solid) and CVaR$_{0.99}$ criterion
(dashed) as a function of deployment scale $T$ (Moderate scenario,
Model~A). For $T \gtrsim 10^5$, the two criteria yield
essentially identical optima. At smaller $T$, Monte Carlo noise
($M_{\mathrm{sim}} = 10{,}000$) produces fluctuations of
$\pm 0.10$ in the CVaR estimate, but no systematic divergence is
observed.}
\label{fig:cvar-comparison}
\end{figure}

Under Model~B (multiplicative Pareto damage), the CVaR-optimal
$\alpha^*$ is essentially invariant to the tail exponent
$\alpha_{\mathrm{PL}}$. Table~\ref{tab:cvar-alpha-pl} shows that
$\alpha^*_{\mathrm{CVaR}} = 0.50$ across all tested values of
$\alpha_{\mathrm{PL}}$; only the CVaR magnitude scales with tail
heaviness (approximately $2.5\times$ from
$\alpha_{\mathrm{PL}} = 3.0$ to $1.5$).

\begin{table}[h]
\centering
\small
\begin{tabular}{lccl}
\toprule
$\alpha_{\mathrm{PL}}$ & $\alpha^*_{\mathrm{CVaR}}$
  & CVaR & 95\% CI \\
\midrule
3.0 & 0.50 & 925  & $[873,\; 962]$ \\
2.5 & 0.50 & 1037 & $[977,\; 1077]$ \\
2.0 & 0.50 & 1288 & $[1209,\; 1329]$ \\
1.5 & 0.50 & 2352 & $[2141,\; 2592]$ \\
\bottomrule
\end{tabular}
\caption{\textbf{CVaR sensitivity to tail exponent
$\alpha_{\mathrm{PL}}$ (Moderate scenario, $T = 10^4$). The optimum
$\alpha^*$ is unchanged; only the CVaR magnitude scales with tail
heaviness.}}
\label{tab:cvar-alpha-pl}
\end{table}

This invariance arises because the Pareto context multiplier $X$
enters multiplicatively and independently of~$\alpha$: it scales
harm uniformly across all design points, preserving their relative
ranking. Within the class of $\alpha$-separable tail models, the
optimal safety design is insensitive to both the choice of risk
criterion (expected harm vs.\ CVaR) and the heaviness of the damage
tail. What changes is the \emph{magnitude} of catastrophic risk,
not the \emph{policy} that minimizes it. We discuss the scope and
limitations of this invariance in the Discussion.

\subsection{Harm Decomposition: The Role of
CMF}\label{sec:decomposition}

Figure~\ref{fig:decomposition} decomposes $E_{\mathrm{harm}}$ into
its normal-operation and CMF components at $T = 10^6$. At
low~$\alpha$, the CMF component (purple) is visible as a
non-negligible share of total harm: because $g_\alpha \approx g_0$
when character shaping is minimal, the system remains exposed when
filters are disabled during CMF. At moderate-to-high~$\alpha$,
$g_\alpha \ll g_0$ reduces post-CMF harm, and the CMF share
diminishes accordingly. Despite this structural role, the CMF
probability parameter $e_0$ has negligible influence on the optimum
($\Delta\alpha^* = +0.01$ over a 19-fold range in sensitivity
analysis), indicating that the dominant channel determining
$\alpha^*$ is the fragility--filter tradeoff in
$L_{\mathrm{normal}}$, not the CMF--character channel in
$L_{\mathrm{CMF}}$.

\begin{figure}[!htbp]
\centering
\includegraphics[width=\columnwidth]{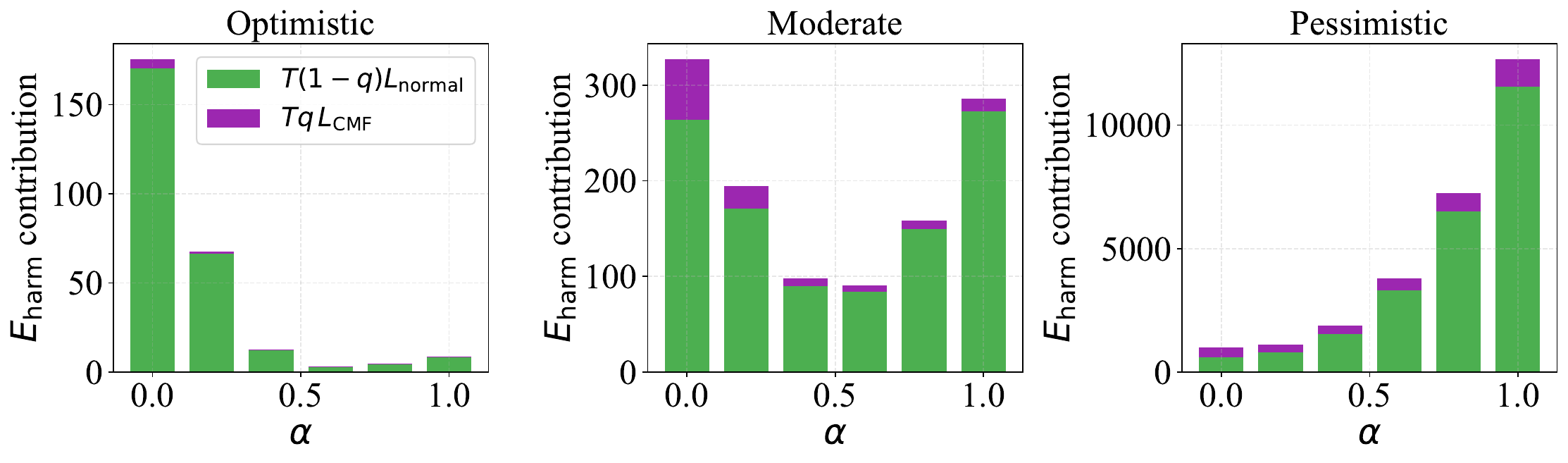}
\caption{\textbf{Decomposition of expected harm into normal-operation
and CMF components.} Stacked bars show the contribution of
$T(1-q)L_{\mathrm{normal}}$ (green) and $Tq\,L_{\mathrm{CMF}}$
(purple) to total $E_{\mathrm{harm}}$ at $T = 10^6$ for selected
values of $\alpha$ across all three scenarios. The CMF share is
largest at low $\alpha$ and diminishes at higher $\alpha$ as
character shaping reduces post-CMF harm.}
\label{fig:decomposition}
\end{figure}

\FloatBarrier

\section{Discussion}\label{sec:discussion}

The central finding of our analysis is that the baseline character
fragility rate $p_{\mathrm{frag}}^{(0)}$ dominates the optimal
safety design, moving $\alpha^*$ across a range of $0.50$ --- far
exceeding the influence of deployment scale, tail severity, or any
other model parameter (Table~\ref{tab:sensitivity}). This
dominance reflects a structural asymmetry: $p_{\mathrm{frag}}^{(0)}$
simultaneously increases the cost of high-$\alpha$ designs (through
fragility events) and decreases their benefit (through degraded
post-fragility behavior), creating a double penalty that no other
parameter imposes. In the subsections that follow, we interpret the
remaining results and discuss their limitations.

\subsection{Why Less Character Investment Can Suffice When Shaping
Is Effective}

A counterintuitive finding from the sensitivity analysis is that
stronger character-shaping capability (larger $\Delta\mu$, smaller
$r_\sigma$) is associated with \emph{lower} $\alpha^*$.
Specifically, increasing $\Delta\mu$ from $0.2$ to $2.0$ reduces
$\alpha^*$ from $0.67$ to $0.40$; decreasing $r_\sigma$ from $1.0$
to $0.4$ reduces $\alpha^*$ from $0.62$ to $0.41$.

This result reflects \emph{diminishing returns} on character
shaping. When $\Delta\mu$ is large, even modest $\alpha$ (e.g.,
$\alpha \approx 0.3$--$0.4$) achieves most of the achievable
reduction in $g_\alpha$. Beyond this point, additional investment
in character shaping yields minimal further benefit in expected harm
reduction while incurring growing fragility costs through
$p_{\mathrm{frag}}(\alpha) = p_{\mathrm{frag}}^{(0)} \cdot
\alpha^n$.

The practical implication is encouraging: a safety architecture
does not need to maximize character-shaping investment to capture
the bulk of its value. A moderate allocation to character shaping,
combined with continued investment in runtime filters, may yield a
better risk profile than an aggressive push toward full character
reliance.

We note that this result is partially a consequence of the
functional forms chosen: the Gaussian tail of $g_\alpha$ decays
exponentially in $\alpha$, while the fragility cost
$p_{\mathrm{frag}}(\alpha) = p_{\mathrm{frag}}^{(0)} \cdot
\alpha^n$ grows polynomially. This asymmetry mechanically favors
low $\alpha$ when the exponential decay is fast (large $\Delta\mu$).
Under alternative functional forms --- for instance, if fragility
also decayed exponentially beyond a threshold --- the diminishing-returns
effect could weaken or reverse. Our functional-form sensitivity
analysis (varying $n \in \{0.5, 1, 2, 3\}$) confirms that the
qualitative direction persists across power-law fragility costs,
but the quantitative magnitude of the effect is form-dependent.

This finding is consistent with the broader engineering
principle that robustness comes from diversified defenses rather
than from maximizing any single layer.

\subsection{The Scale Effect Is Real but Regime-Dependent}

The monotonicity result $\Delta\alpha^* \geq 0$ across the
explored parameter space (1{,}200 grid cells, zero negative) is a
structural property of the current model rather than a contingent
empirical finding. In our formulation, deployment scale $T$ enters
exclusively through $M = \rho_{\mathrm{edge}} \cdot T$, which
degrades runtime filters (via $\varepsilon(\alpha, M)$) and raises
CMF probability (via $q(M)$). Both channels penalize low-$\alpha$
(rule-heavy) designs. Crucially, $T$ has no channel through which
it degrades character shaping: $p_{\mathrm{frag}}(\alpha)$ and
$g_\alpha$ are $T$-independent. The monotonicity
$\Delta\alpha^* \geq 0$ therefore follows from the model's
asymmetric treatment of scale effects on the two safety layers.

As discussed in the Formal Framework,
$p_{\mathrm{frag}}$ is an intrinsic property of the trained model
--- the fraction of vulnerable input space --- and is
$T$-independent by construction. An alternative modeling choice
would treat $p_{\mathrm{frag}}$ as effectively $T$-dependent, for
instance if expanding deployment to new user populations shifts the
effective input distribution further from training, exposing
previously untested regions. Under such a reinterpretation, the
monotonicity $\Delta\alpha^* \geq 0$ would no longer be
structurally guaranteed. We regard the fixed-vulnerability
interpretation as appropriate for analyzing a given model in a
given deployment context, while the shifting-distribution extension
addresses a different question --- how safety architecture should
adapt when deployment expands across heterogeneous contexts ---
and is left to future work.

Despite this structural origin, the scale effect is quantitatively
weak in the Optimistic scenario ($\Delta\alpha^* = +0.01$). This
suggests that when character-shaping technology is mature and
fragility risk is low, the scaling law effectively disappears: the
optimal design becomes insensitive to deployment scale. This
finding carries a policy implication --- the urgency of adjusting
safety architecture as systems scale depends critically on how
fragile current character-shaping methods are. If fragility can
be driven below approximately~5\%, the scaling question becomes
moot.

\subsection{Robustness to Risk Criterion and Tail Severity}

The convergence of CVaR-based and expected-harm-based optima at
$T \gtrsim 10^5$, and the invariance of $\alpha^*$ to the Pareto
tail exponent $\alpha_{\mathrm{PL}}$, together constitute a
robustness result: within the class of $\alpha$-separable tail
models, the optimal safety design is insensitive to both the
choice of risk criterion and the heaviness of the damage
distribution. This follows structurally from the multiplicative
independence of the context multiplier~$X$ and~$\alpha$, which
preserves the relative ranking of designs across
$\alpha_{\mathrm{PL}}$ values. What changes is the
\emph{magnitude} of catastrophic risk, not the \emph{policy} that
minimizes it.

We note, however, that this $\alpha$-separability is likely a
simplification. In real systems, different design regimes may
produce qualitatively different failure modes with distinct tail
structures, breaking the independence between $X$ and $\alpha$.
If the tail exponent were itself $\alpha$-dependent, the
CVaR-optimal $\alpha^*$ could diverge from the expected-harm
optimum. Developing models that capture such design-dependent tail
behavior is an important direction for future work.

\subsection{The Quantitative Role of Common-Mode Failure}

Although the CMF channel is quantitatively second-order at the
optimum (see Harm Decomposition above), it plays an essential
\emph{qualitative} role: it is the mechanism that makes rules-only
designs ($\alpha = 0$) sharply suboptimal at large~$T$. When CMF
disables the entire filter layer, systems with $\alpha = 0$ lose
all protection (since $g_0$ is large), whereas systems with
$\alpha > 0$ retain the ``inner wall'' of character shaping
($g_\alpha < g_0$) --- the asymmetry that makes $\alpha^* > 0$ at
large~$T$. This may also reflect our independence assumption:
correlated failure modes that simultaneously disable filters and
trigger character fragility could amplify the CMF channel's
importance, an extension we leave to future work.

\subsection{Limitations}

Several simplifications limit the scope of our conclusions. The
action space is one-dimensional and Gaussian, capturing the
safety--harm axis but discarding multi-attribute structure and
heavy-tailed action behavior that could alter $g_\alpha$. The model
is static: it compares equilibria across scales without capturing
adversary adaptation, filter updates, or fragility evolution over
time. The $\alpha$ parameterization treats training-time and
inference-time effort as a shared resource pool and is best read
as a proxy for relative emphasis rather than a literal budget; its
three coupled roles --- training allocation, deployment-time
dependence on shaping, and fragility susceptibility --- may be
partially independent in practice. Tail risk is captured only on
the severity side: dependence-side heavy tails, where incident
occurrences cluster in time, are not modeled and would carry
different design implications. Finally, key parameters ---
particularly $p_{\mathrm{frag}}^{(0)}$ and $\alpha_{\mathrm{PL}}$
--- lack direct empirical calibration, which our scenario-based
approach mitigates but does not eliminate; empirical estimation of
character fragility under distributional shift remains an important
direction for future work.

\section{Conclusion}

This paper introduced a stylized comparative-statics model for
analyzing the optimal balance between character shaping and rule
enforcement in AI safety design as a function of deployment scale.
Our analysis yields a clear and actionable conclusion: the most
consequential determinant of optimal safety architecture is the
baseline character fragility rate $p_{\mathrm{frag}}^{(0)}$ ---
the probability that shaped safe behavior degrades or collapses
under novel conditions. This single parameter moves the optimal
design across a range of $0.50$, from strongly character-reliant
to strongly rule-reliant, far exceeding the influence of deployment
scale, tail severity, filter quality, or common-mode failure
probability.

This finding carries direct implications for the research agenda
in AI safety. If character fragility is the dominant lever, then
three priorities follow.

First, \textbf{measuring fragility is essential}. At present, no
standardized metric exists for quantifying how reliably a model's
shaped behavior generalizes beyond its training distribution. Our
results show that even order-of-magnitude uncertainty in
$p_{\mathrm{frag}}^{(0)}$ --- e.g., whether it is 1\% or 10\% ---
leads to qualitatively different optimal designs. Developing
rigorous, reproducible benchmarks for character fragility under
distributional shift is therefore a prerequisite for principled
safety architecture decisions.

Second, \textbf{reducing fragility may be more valuable than
improving either character shaping or filter quality in isolation}.
Our sensitivity analysis shows that improving filter technology
($\varepsilon_{\min}$) shifts $\alpha^*$ by only $+0.07$, and
increasing character-shaping strength ($\Delta\mu$) exhibits
diminishing returns. By contrast, reducing
$p_{\mathrm{frag}}^{(0)}$ from $0.10$ to $0.01$ shifts $\alpha^*$
by approximately $+0.30$, unlocking a fundamentally different and
more efficient safety regime. Research on training methods that
produce robust generalization --- including adversarial training
for distributional robustness, mechanistic interpretability to
identify fragile internal representations, and evaluation
protocols that systematically probe out-of-distribution behavior
--- is therefore of first-order importance.

Third, \textbf{the scaling question becomes moot if fragility is
sufficiently low}. Under our Optimistic scenario
($p_{\mathrm{frag}}^{(0)} = 0.01$), the optimal design is
essentially invariant to deployment scale
($\Delta\alpha^* = +0.01$ across six orders of magnitude in $T$).
This suggests that the urgency of adapting safety architecture to
scale is itself contingent on the current state of
character-shaping reliability. If fragility can be driven below
approximately 5\%, the design question simplifies from ``how
should safety change as we scale?'' to ``what is the right hybrid
at any scale?'' --- and the answer is stable.

Our model also establishes several structural results that hold
across the explored parameter space: the optimum is never pure
character ($\alpha^* < 1$ in all cases), and is a hybrid in most
regimes; the optimal character weight is
weakly non-decreasing in deployment scale; the optimal policy is
robust to the choice of risk criterion (expected harm vs.\ CVaR);
and stronger character-shaping capability is associated with
lower optimal $\alpha^*$ due to diminishing returns.

These findings should be interpreted within the limitations of a
stylized model. The one-dimensional action space, Gaussian
behavioral distribution, static analysis, and
$\alpha$-independent tail structure are simplifications that
future work should relax. In particular, models in which
deployment scale also increases character fragility could yield
reversed scaling predictions, and design-dependent tail structures
could break the CVaR--expected-harm equivalence observed here.

Nevertheless, the core message is clear. The most important
question for AI safety design is not how large a system will be
deployed, nor how severe the worst-case damage might be, but how
reliably the system's shaped character holds when it encounters
conditions its designers did not foresee. Answering this question
--- through measurement, through training methodology, and through
rigorous evaluation --- is the most direct path to safety
architectures that scale.
The reassuring news is that the answer does not depend on
how many patients the model is serving --- the only thing
a patient should have to worry about is whether the one
they are asking can be trusted.
\label{TEMP-bodyend}

\bibliography{aaai2026}

\appendix

\section{Appendix A: Scenario Parameters}

Table~\ref{tab:scenario_params} lists all scenario-varying
parameters. Parameters constant across scenarios are:
$\mu_0 = 0.0$, $\sigma_0 = 1.0$, $\tau = -2.0$, $k = 1.0$,
$M_{\mathrm{ref}} = 10^6$, $n_{\mathrm{frag}} = 2$.

\begin{table}[ht]
\centering
\caption{Scenario parameters. Parameters constant across all
scenarios are listed in the text above.}
\label{tab:scenario_params}
\small
\begin{tabular}{llccc}
\toprule
Group & Parameter & Opt. & Mod. & Pess. \\
\midrule
\multirow{2}{*}{Character shaping}
  & $\Delta\mu$       & 1.5  & 1.0  & 0.5  \\
  & $r_\sigma$        & 0.6  & 0.7  & 0.9  \\
\midrule
\multirow{3}{*}{Filter quality}
  & $\varepsilon_{\min}$          & 0.02 & 0.03 & 0.05 \\
  & $\varepsilon_{\max,\mathrm{base}}$ & 0.10 & 0.15 & 0.25 \\
  & $\varepsilon_{\mathrm{ceiling}}$   & 0.25 & 0.30 & 0.50 \\
\midrule
\multirow{3}{*}{Deployment scale}
  & $\rho_{\mathrm{edge}}$ & 0.01 & 0.05 & 0.10 \\
  & $\beta_d$              & 0.5  & 1.0  & 2.0  \\
  & $d_0$                  & 0.05 & 0.10 & 0.30 \\
\midrule
\multirow{2}{*}{CMF}
  & $\beta_q$         & 0.3  & 0.5  & 1.0  \\
  & $e_0$             & 0.2  & 0.3  & 0.5  \\
\midrule
\multirow{4}{*}{Fragility}
  & $p_{\mathrm{frag}}^{(0)}$          & 0.01 & 0.03 & 0.10 \\
  & $\mu_{\mathrm{frag}}$              & 0.0  & $-0.5$ & $-1.0$ \\
  & $\sigma_{\mathrm{frag}}$           & 1.0  & 1.2  & 1.5  \\
  & $\varepsilon_{\mathrm{frag,factor}}$ & 1.0  & 1.0  & 2.0  \\
\midrule
Tail index
  & $\alpha_{\mathrm{PL}}$ & 3.0  & 2.5  & 2.0  \\
\bottomrule
\end{tabular}
\end{table}

\end{document}